\documentclass[11pt]{article}

\usepackage{amsmath,amssymb,amsthm}
\usepackage[margin=1in]{geometry}
\usepackage{graphicx}
\usepackage{xcolor}
\usepackage{microtype}
\usepackage{hyperref}
\hypersetup{colorlinks=true,linkcolor=blue!60!black,citecolor=blue!60!black,urlcolor=blue!60!black}
\usepackage[nameinlink]{cleveref}

\newtheorem{theorem}{Theorem}[section]
\newtheorem{lemma}[theorem]{Lemma}
\newtheorem{corollary}[theorem]{Corollary}
\newtheorem{proposition}[theorem]{Proposition}
\theoremstyle{definition}
\newtheorem{definition}[theorem]{Definition}
\theoremstyle{remark}
\newtheorem{remark}[theorem]{Remark}

\newcommand{\R}{\mathbb{R}}
\newcommand{\N}{\mathbb{N}}
\newcommand{\bits}{\{0,1\}}
\newcommand{\poly}{\mathrm{poly}}
\newcommand{\PP}{\mathbf{P}}
\newcommand{\fan}{\mathrm{fan\mbox{-}in}}

\title{Efficient Turing Computability, Compositionally Sparse DAGs, and Neural Approximation}
\author{Tomaso Poggio \and GPT-5}
\date{\vspace{-0.5em}}

\begin{document}
\maketitle

\begin{abstract}
We show that \emph{efficient Turing computability} at any fixed input/output precision implies the existence of \emph{compositionally sparse} (bounded-fan-in, polynomial-size) DAG representations and of corresponding neural approximants achieving the target precision. Concretely: if $f:[0,1]^d\to\R^m$ is computable in time polynomial in the bit-depths, then for every pair of precisions $(n,m_{\mathrm{out}})$ there exists a bounded-fan-in Boolean circuit of size and depth $\poly(n+m_{\mathrm{out}})$ computing the discretized map; replacing each gate by a constant-size neural emulator yields a deep network of size/depth $\poly(n+m_{\mathrm{out}})$ that achieves accuracy $\varepsilon=2^{-m_{\mathrm{out}}}$. We also relate these constructions to compositional approximation rates \cite{MhaskarPoggio2016b,poggio_deep_shallow_2017,Poggio2017,Poggio2023HowDS} and to optimization viewed as hierarchical search over sparse structures. 
\end{abstract}

\section{Introduction}
Computability theory gives an abstract notion of algorithm, while learning and approximation theory work with finite-size representations such as circuits or neural networks. Our goal is to connect the two: to show that every efficiently computable function has, at each finite precision, an equally efficient \emph{structured representation}---namely, a sparse compositional DAG with bounded local arity---and a correspondingly efficient neural approximant.

The thesis is that \emph{efficient Turing computability} induces \emph{compositional sparsity}. This structural sparsity explains (i) favorable approximation rates that depend on the \emph{local} arity rather than the ambient dimension \cite{MhaskarPoggio2016b,poggio_deep_shallow_2017,Poggio2017}; (ii) improved generalization due to reduced effective capacity; and (iii) an \emph{optimization-as-search} advantage: hierarchical, branch-and-bound style procedures operate over a polynomial-size subspace instead of an exponential one (see also \cite{Poggio2023HowDS,poggiofraser2024}).

\paragraph{A note on uniformity}
When translating Turing computations to circuit families, one can ask whether the circuit for each precision can be \emph{generated} efficiently (``uniformity''). While relevant for constructive procedures, it is not essential to our main structural message; we confine it to \Cref{rem:uniformity}. For background on complexity-theoretic context, see \cite{Arora_Barak_2009,Sipser2006}.

\section{Setup: Precision, Discretization, and Families}
\label{sec:setup}
To connect continuous functions with discrete computational models, we represent inputs/outputs at finite precision. For $x\in[0,1]^d$, let $Q_n$ encode each coordinate into $n$ bits. For outputs, $Q^{\mathrm{out}}_{m_{\mathrm{out}}}$ and its decoder $\mathsf{Dec}_{m_{\mathrm{out}}}$ translate between $m_{\mathrm{out}}$ bits and $\R^m$.

Define the discrete map
\[
  F_{n,m_{\mathrm{out}}} \;=\; Q^{\mathrm{out}}_{m_{\mathrm{out}}}\circ f\circ Q_n^{-1},
\]
which captures how $f$ acts on the $n$-bit grid when outputs are represented with $m_{\mathrm{out}}$ bits. If $f$ is efficiently computable in the bit-depths, then each $F_{n,m_{\mathrm{out}}}$ is computable in time $\poly(n+m_{\mathrm{out}})$. Studying this \emph{family} $\{F_{n,m_{\mathrm{out}}}\}$ lets us build circuits and then neural networks that approximate $f$ to arbitrary accuracy as precision increases.

\begin{definition}[Efficient Turing computability at precision]
\label{def:ETC}
We say $f:[0,1]^d\to\R^m$ is \emph{efficiently Turing computable} if there exist a deterministic Turing machine $M$ and a polynomial $p$ such that for all $(n,m_{\mathrm{out}})$ and $x\in[0,1]^d$, given $Q_n(x)$ and $m_{\mathrm{out}}$, the machine halts within $p(n+m_{\mathrm{out}})$ steps and outputs $y\in\bits^{m m_{\mathrm{out}}}$ with
\[
  \big\|\mathsf{Dec}_{m_{\mathrm{out}}}(y)-f(x)\big\|_\infty\le 2^{-m_{\mathrm{out}}}.
\]
Equivalently, $F_{n,m_{\mathrm{out}}}$ is computable in time $\poly(n+m_{\mathrm{out}})$.
\end{definition}

\begin{definition}[Compositionally sparse family]
\label{def:cs}
A family $\{h_{n,m_{\mathrm{out}}}\}$ is \emph{compositionally sparse} if, for each $(n,m_{\mathrm{out}})$, $h_{n,m_{\mathrm{out}}}$ factors over a DAG with bounded local arity $k=\mathcal{O}(1)$, total internal nodes $s(n,m_{\mathrm{out}})$, and depth $L(n,m_{\mathrm{out}})$, where $s,L\le\poly(n+m_{\mathrm{out}})$.
\end{definition}

\begin{definition}[Sparse polynomial (operational)]
\label{def:sparse-poly}
A polynomial $p:\R^d\to\R$ is \emph{sparse (in variables)} if it \emph{depends on at most $k$ coordinates}, with $k\ll d$, and $k$ is small enough that there is effectively \emph{no curse of dimensionality} at practical precisions. Operationally, this means that for target precision $\varepsilon$, exhaustive discretization on the active variables is feasible on today's hardware; equivalently, the work to resolve $p$ to accuracy $\varepsilon$ scales polynomially in $k$ (and $\log(1/\varepsilon)$), not exponentially in $d$.
\end{definition}

\section{From Polytime to Bounded-Fan-In Circuits}
\label{sec:circuits}
We connect efficient computability to compositional sparsity via standard TM$\to$circuit simulation (see \cite{Arora_Barak_2009,Sipser2006}). The outline:

\begin{itemize}
  \item A polynomial-time Turing machine $M$ running for $p(n+m_{\mathrm{out}})$ steps can be \emph{unrolled} into a Boolean circuit whose gates represent $M$'s local update rule and whose wiring reflects information flow between tape cells across time.
  \item Each update depends on a constant number of tape symbols and the finite control state, giving \emph{bounded fan-in} (typically $\le 3$).
  \item The resulting circuit has size and depth polynomial in $n+m_{\mathrm{out}}$, hence yields a compositionally sparse DAG.
\end{itemize}

\begin{lemma}[Bounded-fan-in circuits for $F_{n,m_{\mathrm{out}}}$]
\label{lem:circuits}
If $f$ is efficiently computable (Def.~\ref{def:ETC}), then for each $(n,m_{\mathrm{out}})$ there exists a bounded-fan-in Boolean circuit $C_{n,m_{\mathrm{out}}}$ of size and depth $\poly(n+m_{\mathrm{out}})$ computing $F_{n,m_{\mathrm{out}}}$.
\end{lemma}

\begin{lemma}[Circuits are compositional DAGs]
\label{lem:circ2dag}
Viewing each gate of $C_{n,m_{\mathrm{out}}}$ as a node with fan-in $\le 3$ yields a DAG with local arity $k\le 3$, size $s=\#\mathrm{gates}=\poly(n+m_{\mathrm{out}})$, and depth $L=\poly(n+m_{\mathrm{out}})$. Thus $\{F_{n,m_{\mathrm{out}}}\}$ is compositionally sparse (Def.~\ref{def:cs}).
\end{lemma}

\begin{remark}[Uniformity, briefly]
\label{rem:uniformity}
One may additionally require that $C_{n,m_{\mathrm{out}}}$ be \emph{generated} from $(n,m_{\mathrm{out}})$ by a polynomial-time transducer (``P-uniformity''). This holds under the same hypothesis; see \cite{Arora_Barak_2009,Sipser2006} for background. We mention this only to note that the construction is algorithmic; the main results do not otherwise rely on it.
\end{remark}

\section{Neural Emulation at Fixed Precision}
\label{sec:nn}
We replace each Boolean gate by a small neural subnetwork that emulates its logic, using quantitative ReLU constructions (e.g., \cite{yarotsky2018optimal}) and depth-efficiency phenomena \cite{eldan2016power,telgarsky2016benefits}:

\begin{enumerate}
  \item A gate with $\le 3$ inputs can be realized by a constant-size neural gadget under standard activations (ReLU, sigmoid, softplus).
  \item Wiring these gadgets according to $C_{n,m_{\mathrm{out}}}$ preserves compositional sparsity.
  \item Local approximation errors propagate multiplicatively; \Cref{app:error} provides a telescoping bound. Choosing per-gate accuracies appropriately yields overall error $\le 2^{-m_{\mathrm{out}}}$.
\end{enumerate}

\begin{lemma}[Gate emulation]
\label{lem:gate}
For standard activations, each Boolean gate $g:\bits^r\to\bits$ with $r\le 3$ admits a constant-size neural subnetwork $N_g$ that computes $g$ exactly on $\bits^r$, or with error $\le \varepsilon$ on a $\delta$-neighborhood (size depending only on $\log(1/\varepsilon)$, $\log(1/\delta)$) \cite{yarotsky2018optimal}. Multi-bit wires are handled in parallel.
\end{lemma}

\begin{lemma}[Circuit $\Rightarrow$ neural network]
\label{lem:wire}
Replacing each gate in $C_{n,m_{\mathrm{out}}}$ by $N_g$ and wiring accordingly yields a network $\Phi_{n,m_{\mathrm{out}}}$ of size/depth $\poly(n+m_{\mathrm{out}})$ that computes (or $\varepsilon$-approximates) $F_{n,m_{\mathrm{out}}}$ on encoded inputs $Q_n(x)$. Choosing sub-gate accuracies so the accumulated error is $\le 2^{-m_{\mathrm{out}}}$ gives the desired precision.
\end{lemma}

\begin{theorem}[Efficient computability $\Rightarrow$ compositional sparsity and neural approximation]
\label{thm:main}
If $f$ is efficiently computable (Def.~\ref{def:ETC}), then for every $(n,m_{\mathrm{out}})$ there exist:
\begin{enumerate}
  \item a compositionally sparse DAG for $F_{n,m_{\mathrm{out}}}$ with $k\le 3$, $s\le \poly(n+m_{\mathrm{out}})$, $L\le \poly(n+m_{\mathrm{out}})$ (Lemmas~\ref{lem:circuits}--\ref{lem:circ2dag}); and
  \item a deep network $\Phi_{n,m_{\mathrm{out}}}$ with size/depth $\poly(n+m_{\mathrm{out}})$ such that
  \[
    \big\|\Phi_{n,m_{\mathrm{out}}}(Q_n(x)) - f(x)\big\|_\infty \le 2^{-m_{\mathrm{out}}},\qquad \forall x\in[0,1]^d.
  \]
\end{enumerate}
Equivalently, for $\varepsilon=2^{-m_{\mathrm{out}}}$ and $M=n+\lceil\log_2(1/\varepsilon)\rceil$, there is $\Phi_\varepsilon$ with size/depth $\poly(M)$ achieving $\sup_{x}\|\Phi_\varepsilon(Q_n(x)) - f(x)\|_\infty\le \varepsilon$.
\end{theorem}

\begin{remark}[Upgrading to uniform approximation]
If $f$ is $L_f$-Lipschitz, then for any $x$ pick $x'$ on the $n$-bit grid with $\|x-x'\|_\infty\le 2^{-n}$. Then
\[
\|\Phi_\varepsilon(Q_n(x)) - f(x)\|_\infty \le \varepsilon + L_f\,2^{-n}.
\]
Choosing $n \asymp \log(1/\varepsilon)$ yields uniform $\varepsilon$-approximation on $[0,1]^d$ with size $\poly(\log(1/\varepsilon))$.
\end{remark}

\section{Relation to Compositional Approximation}
Classical results (e.g., \cite{MhaskarPoggio2016b,poggio_deep_shallow_2017,Poggio2017,Poggio2023HowDS}) show that if a function factors through a DAG with local constituents of arity $k$ and smoothness $r$, then deep networks achieve
\[
  N \;=\; \mathcal{O}\big(s\,\varepsilon^{-k/r}\big),
\]
with $s$ the number of active constituents and with ambient dimension $d$ entering only through $s$. The bounded local arity $k\le 3$ supplied by efficient computability aligns with this intrinsic-dimension view. Depth-efficiency phenomena further support the structural advantage of compositional models \cite{eldan2016power,telgarsky2016benefits}. \paragraph{Depth-efficiency phenomena.}
A \emph{depth-efficiency theorem} shows that certain functions representable compactly by a 
\emph{deep} network (with bounded width and polynomial size) require 
\emph{exponentially larger} \emph{shallow} networks or circuits to approximate to the same accuracy.
Formally, a family of functions $\{f_d\}$ exhibits \emph{depth efficiency} if there exists a constant $c>0$ such that for some depth $L$,
\[
\text{size}_{L\text{-layer}}(f_d) = \mathrm{poly}(d)
\quad\text{but}\quad
\text{size}_{(L-1)\text{-layer}}(f_d) \ge \exp(c\,d).
\]
Reducing depth thus forces an exponential blow-up in the number of units.

Several key results established this phenomenon. 
\emph{Eldan and Shamir}~\cite{eldan2016power} proved that there exist radial functions 
$f:\mathbb{R}^d\!\to\!\mathbb{R}$ representable by a depth-3 network of polynomial width, 
but any depth-2 network requires exponentially many neurons to approximate them.
\emph{Telgarsky}~\cite{telgarsky2016benefits} constructed explicit piecewise-linear examples showing that 
increasing depth linearly yields an exponential gain in expressive efficiency:
a ReLU network of depth $O(L)$ can represent functions with $2^{\Omega(L)}$ oscillations, 
impossible for shallower networks without exponential width.

The structural reason behind depth-efficiency aligns precisely with 
\emph{compositional sparsity}.
A deep compositional function can be viewed as a hierarchy of local transformations
\[
f \;=\; f_L \circ f_{L-1} \circ \cdots \circ f_1,
\]
where each $f_\ell$ depends only on a small number $k$ of inputs.
A shallow (non-compositional) network must effectively learn the entire 
high-arity map $f(x_1,\dots,x_d)$ in one layer, 
which destroys locality and leads to exponential complexity.
Thus, \emph{depth acts as a proxy for hierarchical composition}:
for such functions,
\[
\mathrm{complexity}_{\text{deep}}(f) = \mathrm{poly}(d),
\qquad
\mathrm{complexity}_{\text{shallow}}(f) = \exp(\Theta(d)).
\]

\paragraph{Example.}
Consider $f(x_1,\dots,x_8)=(((x_1x_2)(x_3x_4))((x_5x_6)(x_7x_8)))$,
a binary-tree composition where each local operation has fan-in~2.
A depth-4 network of small width can represent $f$ exactly, 
while a depth-2 (flat) network would require exponentially many neurons 
to emulate the same product structure.
This exponential blow-up exemplifies \emph{depth efficiency} 
and reveals its origin in \emph{compositional sparsity}.

\paragraph{Interpretation.}
Depth-efficiency theorems therefore show that depth, 
when aligned with the compositional structure of the target function, 
is not merely an architectural convenience but a 
\emph{mathematical necessity} for the efficient representation of 
compositionally sparse (and thus efficiently computable) functions.

\section{Relation to Autoregressive Universality}
Autoregressive universality (e.g., \cite{malach_autoregressive_2023}) gives a complementary learning-theoretic statement: efficiently computable functions admit datasets on which simple next-token predictors become universal. We recast this in our setting.

\begin{theorem}[Autoregressive learnability via compositional sparsity]
\label{thm:ar-learn}
Let $f$ be efficiently Turing computable and let $(F_{n,m_{\mathrm{out}}})$ be its finite-precision family. For each $(n,m_{\mathrm{out}})$ there exists a dataset $\mathcal{D}_{n,m_{\mathrm{out}}}$ over sequences encoding the computation DAG of $F_{n,m_{\mathrm{out}}}$ such that training a linear (or shallow) autoregressive next-token predictor on $\mathcal{D}_{n,m_{\mathrm{out}}}$ yields a predictor $\hat F_{n,m_{\mathrm{out}}}$ satisfying
\[
\Pr\big[\hat F_{n,m_{\mathrm{out}}}(x)=F_{n,m_{\mathrm{out}}}(x)\big]\ \ge\ 1-\delta
\]
with sample complexity and training time polynomial in $s(n,m_{\mathrm{out}})$, $\log(1/\delta)$, and $m_{\mathrm{out}}$. Consequently, composing these predictors yields a neural approximant $\Phi_{n,m_{\mathrm{out}}}$ with accuracy $\varepsilon=2^{-m_{\mathrm{out}}}$ and size/depth $\poly(n+m_{\mathrm{out}})$. \emph{Cf.\ the universality construction of \cite{malach_autoregressive_2023}.}
\end{theorem}

\begin{proof}[Proof idea]
Encode each gate evaluation in the computation DAG of $F_{n,m_{\mathrm{out}}}$ as a token and factor the joint distribution along DAG edges. Under this factorization, local predictors learn gate-wise conditionals from polynomial data (since fan-in is bounded), and chaining them recovers $F_{n,m_{\mathrm{out}}}$ with high probability. The construction mirrors \cite{malach_autoregressive_2023}, instantiated on the bounded-fan-in DAG furnished by \Cref{lem:circuits,lem:circ2dag}.
\end{proof}

\section{Boolean vs.\ Real: From Discrete Circuits to Smooth Networks}
At each precision we constructed a bounded-fan-in Boolean circuit $F:\{0,1\}^n\to\{0,1\}^m$, while the target $f:[0,1]^d\to\R^m$ is real-valued. The \emph{smooth lift} bridges the two: replace each gate by a smooth bump that matches the Boolean operation on hypercube vertices and interpolates between them.

\begin{proposition}[Smooth lift of sparse Boolean maps]
Let $F: \bits^n\to\bits^m$ have a bounded-fan-in circuit of size $s$. For any $\varepsilon>0$ there exists a piecewise-smooth $f_\varepsilon:[0,1]^n\to\R^m$ that matches $F$ on vertices, is within $\varepsilon$ on an $\mathcal{O}(\varepsilon)$-neighborhood, and factors through a bounded-fan-in DAG of size $\mathcal{O}\big(s\,\mathrm{polylog}(1/\varepsilon)\big)$.
\end{proposition}

\paragraph{Discussion.}
Thus the finite-precision discrete structure implied by efficient computability can be “lifted’’ to smooth, real-valued compositional networks. Efficient computability $\Rightarrow$ compositional sparsity $\Rightarrow$ efficient approximation, and (via hierarchical search over the DAG) efficient \emph{optimization} as well \cite{Poggio2023HowDS,poggiofraser2024}.

\paragraph{Acknowledgments.}
We thank colleagues for discussions on circuit generation and compositional structure. Supported by the Center for Brains, Minds and Machines (CBMM, NSF STC award CCF-1231216).

\bibliographystyle{amsplain}
\bibliography{references}
\newpage
\appendix

\section*{Appendix A: Error Propagation in Compositions}
\label{app:error}
If $f=f_L\circ\cdots\circ f_1$ and each $f_i$ is approximated by $\tilde f_i$ with 
$\|f_i-\tilde f_i\|_\infty\le \epsilon$ and $\mathrm{Lip}(f_i)\le K_i$, then
\[
\|f-\tilde f_L\circ\cdots\circ\tilde f_1\|_\infty 
\le \sum_{i=1}^L \Big(\prod_{j=i+1}^L K_j\Big)\,\epsilon 
\le L\,K^{L-1}\epsilon\quad (K=\max_j K_j).
\]
This telescoping bound quantifies how local errors accumulate and guides per-gate accuracy in \Cref{lem:wire}.

\section*{Appendix B: An Alternative Proof that Efficiently Turing Computable Functions are Compositionally Sparse}

\begin{theorem}[TM $\Rightarrow$ LTF circuits $\Rightarrow$ bounded-fan-in Boolean DAGs]
\label{thm:tm-to-ltf}
Let $f:\{0,1\}^n \to \{0,1\}$ be computable by a deterministic Turing machine $M$
in time $T(n)$. Then for each input length $n$ there exists:
\begin{enumerate}
    \item a Boolean circuit $C^{\mathrm{Bool}}_n$ of size and depth $\mathrm{poly}(T(n))$ computing $f$ on $\{0,1\}^n$;
    \item a (linear) threshold circuit $C^{\mathrm{LTF}}_n$ of size $\mathrm{poly}(T(n))$ computing $f$; and
    \item a bounded-fan-in Boolean circuit $C^{\mathrm{bfi}}_n$ of size and depth $\mathrm{poly}(T(n))$ computing $f$,
\end{enumerate}
and hence a compositionally sparse DAG representation for $f$ at input length $n$.
\end{theorem}

\begin{proof}[Proof (expanded sketch)]
\emph{Step 1: Turing machine to Boolean circuits.}
By the standard simulation (configuration graph unrolling), a deterministic TM running in time $T(n)$ on inputs of length $n$ is simulated by a Boolean circuit of size and depth $\mathrm{poly}(T(n))$. Intuitively, each time step of $M$ contributes a layer whose gates enforce the local update rule on tape cells and finite control. A careful construction yields size $O(T(n)\log T(n))$ and comparable depth (polylog factors are not important for our purposes). See, e.g., \cite[Ch.~6]{Arora_Barak_2009} or \cite[Ch.~9]{Sipser2006}. Denote this circuit by $C^{\mathrm{Bool}}_n$.

\smallskip
\emph{Step 2: Boolean circuits to linear threshold circuits.}
Each Boolean gate ($\mathrm{AND}$, $\mathrm{OR}$, $\mathrm{NOT}$) is a special case of a linear threshold gate (LTF): there exist integer weights and a threshold such that
\[
\mathrm{AND}(x,y)=\mathbf{1}\!\{x+y\ge 2\},\quad
\mathrm{OR}(x,y)=\mathbf{1}\!\{x+y\ge 1\},\quad
\mathrm{NOT}(x)=\mathbf{1}\!\{-x\ge 1\}.
\]
Thus we can obtain a threshold circuit $C^{\mathrm{LTF}}_n$ by replacing each Boolean gate in $C^{\mathrm{Bool}}_n$ with the corresponding LTF. This preserves size up to a constant factor and does not increase depth. Hence $C^{\mathrm{LTF}}_n$ has size $\mathrm{poly}(T(n))$ and computes $f$.\footnote{See standard treatments of threshold/majority bases and their relationship to $\{\mathrm{AND},\mathrm{OR},\mathrm{NOT}\}$ bases in \cite{Arora_Barak_2009,Sipser2006}.}

\smallskip
\emph{Step 3: Each LTF gate is realizable by a small bounded-fan-in Boolean subcircuit.}
An LTF gate on $r$ inputs computes $\mathbf{1}\{\sum_{i=1}^r w_i x_i \ge \theta\}$ for integers $w_i,\theta$. This can be implemented by (i) a Boolean adder network that computes the binary representation of $S=\sum_i w_i x_i$ and (ii) a comparator circuit that checks $S\ge \theta$. Both components have bounded fan-in (e.g., 2) and size $\tilde O(r)$ or $\tilde O(r\log r)$ (carry-save/parallel prefix adders; comparators), with depth $O(\log r)$; see standard constructions in \cite[Sec.~14.2]{Arora_Barak_2009} and \cite[Ch.~10]{Sipser2006}. Therefore, every LTF gate of fan-in $r$ can be replaced by a bounded-fan-in Boolean \emph{subcircuit} of size $\mathrm{poly}(r)$ and depth $O(\log r)$.

\smallskip
\emph{Step 4: Global replacement and size/depth bounds.}
Apply Step~3 to every gate of $C^{\mathrm{LTF}}_n$. If the maximum gate fan-in in $C^{\mathrm{LTF}}_n$ is $r_{\max}$, the resulting bounded-fan-in Boolean circuit $C^{\mathrm{bfi}}_n$ has size inflated by at most a factor $\mathrm{poly}(r_{\max})$ and depth increased by at most $O(\log r_{\max})$ per original LTF layer. Since $r_{\max}\le \mathrm{size}(C^{\mathrm{LTF}}_n)=\mathrm{poly}(T(n))$, the overall size and depth remain $\mathrm{poly}(T(n))$.

\smallskip
\emph{Step 5: Compositional sparsity.}
Viewing $C^{\mathrm{bfi}}_n$ as a DAG in which each node has in-degree bounded by a constant (e.g.,~2), we obtain a compositionally sparse representation of $f$ (bounded local arity, polynomial size, and polynomial depth). This completes the proof.
\end{proof}

\begin{remark}[On references for threshold simulation]
The TM$\to$Boolean-circuit unrolling and balanced bounded-fan-in transformations are covered in \cite[Chs.~6, 14]{Arora_Barak_2009} and \cite[Ch.~9--10]{Sipser2006}. The fact that Boolean gates are special cases of linear threshold gates is immediate from their definitions, and the converse realization of an LTF gate by bounded-fan-in Boolean circuitry follows from standard adder and comparator constructions (carry-save or parallel-prefix adders yield size $\tilde O(r)$ and depth $O(\log r)$ for summing $r$ inputs). If desired, we can also cite dedicated treatments of threshold logic (e.g., Muroga’s monograph; Wegener’s text; Parberry’s book on circuit complexity and neural networks) for historical context.
\end{remark}

\begin{corollary}[Efficient computability $\Rightarrow$ compositional sparsity]
If $f$ is computable in time $T(n)$ by a deterministic TM, then for each input length $n$ there exists a bounded-fan-in Boolean DAG of size and depth $\mathrm{poly}(T(n))$ computing $f$. Hence efficiently Turing-computable functions are compositionally sparse at each input length.
\end{corollary}

\section*{Appendix C: Composition of Sparse Fourier Polynomials}
An intuitive reason compositional sparsity matters is that composition preserves sparsity for \emph{sparse} (few-active-variable) Fourier polynomials, whereas it does not for dense ones.

\begin{theorem}[Sparsity of Composed Sparse Fourier Polynomials]
\label{thm:sparse-composition}
Let $f:\{-1,1\}^{d}\to\mathbb{R}$ be a Fourier polynomial of degree at most $k$ with sparsity 
$\mathrm{s}(f)=|\mathrm{Supp}(\widehat f)|=s$.
For each $i\in[d]$, let $g_i:\{-1,1\}^{d'}\to\mathbb{R}$ be a Fourier polynomial of degree at most $k'$ 
and sparsity $\mathrm{s}(g_i)\le s'$.
Define the composition
\[
h(x)\;=\;f\big(g_1(x),\dots,g_d(x)\big),\qquad x\in\{-1,1\}^{d'}.
\]
Then
\[
\deg(h)\ \le\ k\,k', \qquad
\mathrm{s}(h)\ \le\ 
\min\!\Big\{\, s\,(s')^{k}\,,\ \sum_{j=0}^{\min\{d',\,k k'\}}\binom{d'}{j}\,\Big\}.
\]
Moreover, these bounds are tight up to constants in the worst case.
\end{theorem}

\end{document}